\documentclass{article}
\usepackage[preprint]{neurips_2026}

\usepackage[utf8]{inputenc}
\usepackage[T1]{fontenc}
\usepackage{hyperref}
\usepackage{url}
\usepackage{booktabs}
\usepackage{amsfonts}
\usepackage{amsmath}
\usepackage{amssymb}
\usepackage{nicefrac}
\usepackage{microtype}
\usepackage{xcolor}
\usepackage{algorithm}
\usepackage{algpseudocode}
\usepackage{graphicx}
\usepackage{subcaption}
\usepackage{wrapfig}
\usepackage{enumitem}
\usepackage{amsthm}
\usepackage{bm}
\usepackage{placeins}

\setlist{topsep=2pt,parsep=0pt,partopsep=0pt}
\newcommand{\fdot}{\dot{f}}
\newcommand{\SlotFlow}{\textsc{SlotFlow}}
\newcommand{\CLEAN}{\textsc{clean}}
\newcommand{\RR}{\mathbb{R}}
\newcommand{\NN}{\mathcal{N}}
\newcommand{\EE}{\mathbb{E}}
\newcommand{\vtheta}{\boldsymbol{\theta}}
\newcommand{\vx}{\mathbf{x}}
\newcommand{\vc}{\mathbf{c}}

\newcommand{\vh}{\mathbf{h}}
\newcommand{\vq}{\mathbf{q}}
\newcommand{\vk}{\mathbf{k}}
\newcommand{\vv}{\mathbf{v}}
\newcommand{\vs}{\mathbf{s}}

\newtheorem{proposition}{Proposition}
\newtheorem{definition}{Definition}

\title{When Attention Collapses:\\Residual Evidence Modeling for Compositional Inference}

\author{
  Niklas Houba \\
  Institute for Particle Physics and Astrophysics\\
  ETH Zurich\\
  \texttt{nhouba@phys.ethz.ch}
}

\begin{document}

\maketitle

\begin{abstract}
Compositional inference -- the decomposition of observations into an unknown number of latent components -- is central to perception and scientific data analysis. Attention-based models perform well when components are approximately separable, as in object-centric vision. Under \emph{additive superposition}, however -- where multiple components contribute to every observation -- we identify a structural failure mode we term \emph{slot collapse}: multiple slots converge to the same dominant component while weaker ones remain unrepresented. We trace this to a general limitation: attention is \emph{memoryless with respect to explained evidence}. All slots repeatedly operate on the same input without accounting for what has already been explained, so gradients are dominated by the strongest component, inducing shared fixed points across slots. As a result, attention fails to enforce non-redundant allocation under additive superposition. We address this by introducing \emph{residual evidence modeling}, which tracks the remaining explanatory capacity of the input, instantiated via \emph{evidence depletion} -- a minimal modification combining multiplicative depletion with an attention bias. Controlled ablations show that parallel attention, sequential processing alone, and loss-based regularization all fail to resolve collapse; evidence depletion, which adds stateful residual tracking to sequential attention, consistently succeeds. Across synthetic benchmarks and real-world audio mixtures (FUSS), evidence depletion reduces slot collapse by up to an order of magnitude, generalizing beyond synthetic settings. On gravitational-wave source inference for the ESA/NASA LISA mission, under identical architectures, data, and losses, standard attention fails while evidence depletion prevents collapse and enables multi-source posterior estimation. These results show that under additive superposition, residual evidence tracking is the operative ingredient -- preventing collapse and enabling compositional inference.
\end{abstract}

\section{Introduction}
\label{sec:intro}

Many inference problems require decomposing observations into an unknown number of underlying components.
This setting arises across perception and scientific data analysis: object-centric visual reasoning~\citep{Locatello2020,Greff2019}, audio source separation~\citep{Hershey2016}, spectroscopic mixture analysis, and multi-object tracking in cluttered scenes share a common structure -- observations are generated by multiple latent factors whose contributions must be disentangled.

Slot attention~\citep{Locatello2020} provides a powerful mechanism for this problem, representing inputs as a set of latent components that compete via attention~\citep{Singh2022,Elsayed2022,Burgess2019}.
It has been highly successful when components are approximately separable, as in object-centric vision: different objects occupy disjoint spatial regions, and tokens can be uniquely assigned to components.
However, many domains exhibit a fundamentally different structure -- observations are formed through \emph{additive superposition}, where multiple components contribute simultaneously to every token, often with large dynamic range.

In this regime, we observe a systematic failure mode of attention-based decomposition we call \emph{slot collapse}: multiple slots converge to the same dominant component, while weaker components remain unrepresented.
We show this arises from a structural limitation of attention: all slots repeatedly attend to the \emph{same full input}, and attention gradients are dominated by the strongest component, inducing shared fixed points across slots.
Increasing capacity or enforcing sequential processing alone does not resolve collapse: the limitation is not capacity or optimization, but the absence of state over explained evidence.
Standard attention is \emph{memoryless with respect to explained evidence}: the input representation is unchanged across slots and iterations, so nothing prevents redundant allocation.
This motivates our approach: resolving slot collapse under additive superposition by introducing an explicit notion of residual explanatory capacity, which we term \emph{residual evidence modeling}.
We expect this limitation to extend to attention-based decompositions that operate on a fixed input without residual accounting, though we validate it here in the slot attention setting.

We instantiate this principle via \emph{evidence depletion}, a minimal form of stateful attention that enforces competition over residual explanatory evidence.
Each token carries a scalar representing the amount of unexplained signal.
Slots are processed sequentially, and after each slot attends, the available evidence is multiplicatively reduced, restricting subsequent slots to unexplained regions.
All slots observe the same underlying tokens; only the evidence weights change.
As a result, early-slot errors redistribute attention without corrupting the signal, unlike classical iterative subtraction which modifies the data.
A controlled ablation matched on architecture, data, and losses shows that two ingredients jointly drive the improvement: a multiplicative depletion update and an explicit attention bias toward unexplained tokens (Sec.~\ref{sec:synthetic}).

We validate evidence depletion in three settings.
On controlled synthetic benchmarks across three signal types, evidence depletion reduces slot collapse by up to an order of magnitude (Table~\ref{tab:synthetic}).
On real-world audio mixtures from the FUSS benchmark~\citep{FUSS2020}, evidence depletion reduces collapse from $0.29$ to $0.05$, confirming that the mechanism generalizes to natural signals.
On gravitational-wave source inference for the LISA mission~\citep{LISA2017} -- a scientifically demanding problem with high dynamic range and a confusion foreground -- we train three models with identical architecture, data, and losses, varying only the attention mechanism (Table~\ref{tab:lisa_ablation}).
Sequential processing alone is insufficient; evidence depletion prevents collapse and enables posterior estimation.
Combined with conditional normalizing flows~\citep{Papamakarios2021,Rezende2015}, the resulting model (\SlotFlow{}) produces approximately calibrated amortized posteriors.

\paragraph{Contributions.}
We make the following contributions:

\begin{enumerate}[leftmargin=1.5em,itemsep=1pt]
    \item \textbf{Failure mode.}
    We identify \emph{slot collapse} as a structural failure of attention under additive superposition: multiple slots converge to the same dominant component due to shared gradients from a common input (Sec.~\ref{sec:collapse}).

    \item \textbf{Mechanism.}
    We introduce \emph{evidence depletion}, a minimal modification to attention that breaks this symmetry via stateful residual tracking, targeting the failure mode identified in this work under additive superposition.
    Controlled ablations show that multiplicative depletion and an attention bias both contribute substantially to resolving collapse (Sec.~\ref{sec:evidence},~\ref{sec:synthetic}; Prop.~\ref{prop:monotone}; App.~\ref{app:gradient}).

    \item \textbf{Validation.}
    Across synthetic benchmarks, real-world audio mixtures (FUSS), and LISA gravitational-wave inference, we show that sequential attention alone fails while evidence depletion consistently prevents collapse, and that loss-based regularization is insufficient (Sec.~\ref{sec:experiments}).
\end{enumerate}
\section{Method}
\label{sec:method}

\subsection{Problem formulation}

Consider an observation $\vx \in \RR^{C \times N}$ ($C$ channels, $N$ samples) generated by $K$ parametric sources plus noise:
\begin{equation}\label{eq:observation}
    \vx = \sum_{k=1}^{K} \vh_k(\vtheta_k) + \mathbf{n}\,,
    \quad K \sim p(K),\; \vtheta_k \sim p(\vtheta),\; \mathbf{n} \text{ with known statistics.}
\end{equation}
We represent the posterior as a permutation-invariant set of per-source marginals $q_\phi(\vtheta_k \mid \vc_s(\vx))$, matched to ground truth via Hungarian assignment, and train an amortized estimator on simulated data via
\begin{equation}\label{eq:npe}
    \phi^* = \arg\min_\phi \;
    -\EE_{p(K)\,p(\vtheta \mid K)\,p(\vx \mid \vtheta)}\!\left[\min_{\sigma \in \mathcal{S}_K} \sum_{k=1}^K \log q_\phi(\vtheta_k \mid \vc_{\sigma(k)}(\vx))\right],
\end{equation}
where $\sigma$ resolves the source-to-slot assignment via Hungarian matching (Sec.~\ref{sec:hungarian}) and $\vc_s(\vx) \in \RR^{d_c}$ is a learned, slot-specific context.
The challenge is constructing $\{\vc_s\}$ such that each source receives a unique, information-preserving representation.
Slot attention is a natural candidate for this decomposition; we now examine why it fails under additive superposition.

\subsection{Failure of attention under additive superposition}
\label{sec:collapse}

Slot attention~\citep{Locatello2020} decomposes inputs into a set of latent slots via competitive cross-attention.
Given tokens $\{\vh_\ell\}_{\ell=1}^L$ and slots $\{\vs_s\}_{s=1}^S$, both in $\RR^d$, with queries $\vq_s = W_q \vs_s$ and keys $\vk_\ell = W_k \vh_\ell$, attention weights are
\begin{equation}\label{eq:standard_sa}
    \alpha_{s\ell} = \frac{\exp(\vq_s \cdot \vk_\ell / \sqrt{d})}{\sum_{s'} \exp(\vq_{s'} \cdot \vk_\ell / \sqrt{d})}\,,
\end{equation}
with normalization over slots for each token.
The softmax enforces per-token competition: slots compete to explain individual tokens, and attention weights sum to one over slots.
This is a form of \emph{assignment} -- it asks ``which slot should explain this token?'' -- and implicitly assumes that tokens can be partitioned across components.
In object-centric vision, this approximately holds: different objects occupy disjoint spatial regions, and assignment is sufficient.
Under \emph{additive superposition}, however, tokens approximately take the form $\vh_\ell \approx \sum_{k=1}^K \vh_\ell^{(k)}$ (exact in the signal domain; approximate after a nonlinear encoder), with multiple sources contributing simultaneously to each token, so tokens no longer admit a partition into components.
Slot attention enforces competition over the assignment of tokens, but not over the allocation of explanatory capacity; under additive superposition, these are not equivalent.
The competition redirects losing slots to different tokens, but when every token contains the same dominant source, different tokens do not mean different sources.
The input representation remains unchanged across slots and iterations -- there is no memory of explained evidence, so nothing prevents multiple slots from claiming the same component.

Before describing the failure mode, we make the central concept precise.

\begin{definition}[Active slot and slot collapse]\label{def:collapse}
A slot $s$ is \emph{active} on input $\vx$ if its existence-head logit (Sec.~\ref{sec:hungarian}) exceeds 0 (i.e., the slot is predicted to correspond to a real source).
\emph{Slot collapse} occurs on a multi-source input ($K \geq 2$) when two or more active slots concentrate attention on the same component.
We operationalize this via two complementary metrics:
\begin{enumerate}[leftmargin=1.5em,itemsep=0pt]
    \item[(i)] \emph{Peak overlap rate} (synthetic benchmarks): the fraction of multi-source inputs on which $\geq 2$ active slots peak at the same token.
    \item[(ii)] \emph{Max active overlap} (LISA): the maximum pairwise cosine similarity between active slot attention vectors, averaged over inputs.
\end{enumerate}
Both metrics capture the same phenomenon -- non-distinct attention patterns across active slots -- but at different granularities (formal definitions in Appendix~\ref{app:metrics}).
\end{definition}

\paragraph{Gradient dominance and shared fixed points.}
Consider a token containing contributions from two sources with amplitudes $A_1 \gg A_2$.
The attention logits $\vq_s \cdot \vk_\ell$ are dominated by the stronger component, and gradients with respect to slot parameters scale accordingly.
Because all slots attend to the same tokens, the gradient of attention with respect to slot queries is aligned across slots: $\partial \alpha_{s\ell}/\partial \vq_s \propto \alpha_{s\ell}(1 - \alpha_{s\ell}) \cdot \vk_\ell$, where $\vk_\ell = W_k \vh_\ell$ is independent of $s$.
All slots are therefore driven toward the same dominant component, inducing \emph{shared fixed points} that persist under increased capacity and sequential processing.

\paragraph{Toy analysis.}
To build intuition (App.~\ref{app:gradient} gives a full gradient analysis), consider two sources contributing additively: $\vh_\ell = A_1 \vh^{(1)} + A_2 \vh^{(2)}$ with $A_1 \gg A_2$.
Attention logits $\vq_s \cdot \vk_\ell \propto A_1 \phi^{(1)} + A_2 \phi^{(2)}$ are dominated by the $A_1$ term, and since all slots are initialized from a shared prior and observe the same tokens, they receive gradient updates pointing in the same $A_1$-dominated direction -- yielding shared fixed points.
This symmetry is induced by the shared input, not by initialization, capacity, or training procedure, and therefore persists under different seeds or sequential processing.
\emph{Attention without residual accounting fails to enforce non-redundant allocation under additive superposition.}
When source amplitudes are comparable, the gradient dominance argument weakens; empirically, evidence depletion still prevents collapse in this regime (our synthetic benchmarks use only $5$--$10\times$ amplitude ratios), but the theoretical motivation is strongest under large dynamic range.

\subsection{Sequential attention with evidence depletion}
\label{sec:evidence}

To address this limitation, we introduce a mechanism that explicitly models residual explanatory capacity.
Each token $\ell$ is associated with a scalar $e_\ell \in [\epsilon, 1]$, initialized to $e_\ell = 1$, representing unexplained evidence.
Slots are processed sequentially; for slot $s$, attention is computed as
\begin{equation}\label{eq:evidence_attn}
    \alpha_{s\ell} = \mathrm{softmax}_\ell\!\left(\frac{\vq_s \cdot \vk_\ell}{\tau\sqrt{d}} + \gamma \log e_\ell\right), \qquad \vk_\ell = W_k(e_\ell \cdot \vh_\ell),\; \vv_\ell = W_v(e_\ell \cdot \vh_\ell).
\end{equation}
Note the shift from Eq.~\eqref{eq:standard_sa}: because slots are processed sequentially, the softmax now normalizes over \emph{tokens} (each slot selects which tokens to attend to) rather than over slots.
Evidence enters through two coupled pathways: the additive $\log e_\ell$ bias suppresses explained tokens, while the scaled inputs $e_\ell \cdot \vh_\ell$ reduce information content at claimed locations.
After slot $s$ attends, evidence is updated multiplicatively:
\begin{equation}\label{eq:evidence_update}
    e_\ell \leftarrow \max\!\bigl(e_\ell \cdot (1 - \alpha_{s\ell}^2),\; \epsilon\bigr)\,,
\end{equation}
inducing monotonic reduction of available evidence at attended tokens (Eq.~\ref{eq:evidence_update} shows the quadratic form; we also evaluate linear $(1{-}\alpha)$, cubic, and binary variants in Sec.~\ref{sec:synthetic}).

\paragraph{Properties.}
All slots observe the same underlying signal; only the evidence weights change, so errors in early slots redistribute explanatory capacity but do not directly corrupt the input signal seen by later slots.
The mechanism is related to but structurally distinct from matching pursuit~\citep{Mallat1993} and \CLEAN{}~\citep{Hogbom1974}: those methods modify the data through subtraction, so estimation errors propagate into subsequent iterations as corrupted input; evidence depletion preserves the signal and modulates only attention allocation.
In the hard-attention limit ($\gamma \to \infty$), evidence depletion reduces to progressive token masking -- analogous to but distinct from classical subtraction, which removes only the estimated component while leaving other superposed components intact.
The multiplicative update ensures that tokens which strongly influence one slot cannot dominate subsequent slots, breaking the gradient symmetry identified in Sec.~\ref{sec:collapse} (we discuss when this evidence-space formulation matters empirically in Sec.~\ref{sec:synthetic}).

\begin{proposition}[Residual separation]\label{prop:monotone}
Let $e_\ell^{(s)}$ denote the evidence at token $\ell$ after processing $s$ slots. Then:
(i)~$e_\ell^{(s+1)} \leq e_\ell^{(s)}$ (monotonic decay);
(ii)~in the limit $S \to \infty$, tokens with persistent attention ($\sum_s \alpha_{s\ell}^2 = \infty$) satisfy $e_\ell^{(s)} \to \epsilon$; for finite $S$, strong attention achieves substantial depletion in few steps;
(iii)~the effective attention distribution for subsequent slots differs whenever $\alpha_{s\ell} > 0$ for any $\ell$.
\end{proposition}
\begin{proof}
(i) follows from $(1-\alpha^2) \in [0,1]$. (ii) follows from $\log(1-x) \leq -x$, giving $\log e_\ell^{(s)} \leq -\sum_{j<s}\alpha_{j\ell}^2 \to -\infty$. (iii) is a direct consequence of (i): the effective input changes after each depletion step.
\end{proof}

These properties ensure each slot operates on a strictly different effective input.
In contrast to standard attention, where identical inputs induce identical gradients and shared solutions, evidence depletion enforces progressive specialization.
Crucially, this breaks the gradient dominance mechanism of Sec.~\ref{sec:collapse}: once a token has been strongly attended, its influence on subsequent slots is progressively suppressed, so dominant components cannot repeatedly attract multiple slots.
A qualitative visualization of this effect is shown in Figure~\ref{fig:evidence} (Appendix~\ref{app:visualization}).

Algorithm~\ref{alg:sequential_sa} summarizes the full procedure.

\begin{algorithm}[t]
\caption{Sequential slot attention with evidence depletion}
\label{alg:sequential_sa}
\begin{algorithmic}[1]
\Require Tokens $\{\vh_\ell\}_{\ell=1}^L$, $S$ slots, $I$ refinement iterations
\State $\vs_s \sim \NN(\bm{\mu}, \bm{\sigma}^2)$ for $s = 1, \ldots, S$
\State $e_\ell \leftarrow 1$ for $\ell = 1, \ldots, L$ \Comment{All signal unexplained}
\State $\pi \leftarrow \text{RandPerm}(S)$ \Comment{Random order (training); fixed canonical order at inference}
\For{$s \in \pi$}
    \For{$i = 1, \ldots, I$}
        \State $\vq_s \leftarrow W_q\,\vs_s$; $\vk_\ell \leftarrow W_k(e_\ell \cdot \vh_\ell)$; $\vv_\ell \leftarrow W_v(e_\ell \cdot \vh_\ell)$
        \State $\alpha_{s\ell} \leftarrow \mathrm{softmax}_\ell\!\bigl(\vq_s \cdot \vk_\ell / (\tau\!\sqrt{d}) + \gamma \log e_\ell\bigr)$
        \State $\vs_s \leftarrow \text{GRU}(\textstyle\sum_\ell \alpha_{s\ell}\,\vv_\ell,\; \vs_s) + \text{MLP}(\vs_s)$
    \EndFor
    \State $e_\ell \leftarrow \max(e_\ell(1 - \alpha_{s\ell}^2),\;\epsilon)$ \Comment{Deplete (quadratic; see Sec.~\ref{sec:synthetic} for variants)}
\EndFor
\State \Return slots $\{\vs_s\}$, attention $\{\alpha_{s\ell}\}$, evidence $\{e_\ell\}$
\end{algorithmic}
\end{algorithm}

\paragraph{Slot ordering.}
At training time, slot processing order is randomized via \texttt{RandPerm} in Algorithm~\ref{alg:sequential_sa} so that the model does not specialize slots to a particular index.
At inference time we use a fixed canonical order (slot index $1, \ldots, S$).
Because Hungarian matching (Sec.~\ref{sec:hungarian}) makes the loss permutation-invariant over slots, training-time randomization does not change the objective; it only diversifies the gradient signal each slot receives over the depletion sequence.

\noindent
Sequential processing increases the critical path from $O(I)$ (vanilla, parallel over slots) to $O(S \cdot I)$; in practice, with $S{=}5$ and $I{=}3$, this is a modest overhead.
The above defines how slots are allocated.
To use slot representations for downstream inference, we integrate the mechanism into a standard amortized inference pipeline.

\subsection{Integration into amortized inference}
\label{sec:hungarian}

The contribution of this work is the attention mechanism; the components below are standard and the mechanism is independent of the downstream prediction head.
Each slot's context $\vc_k$ is converted into a posterior via a conditional normalizing flow~\citep{Rezende2015,Papamakarios2021}: an invertible map $f_\phi: \RR^D \to \RR^D$ transforms a base Gaussian into $p_\phi(\vtheta \mid \vc)$ using rational-quadratic spline coupling layers~\citep{Durkan2019}.
Since slot order is arbitrary, we use Hungarian matching~\citep{Kuhn1955,Carion2020} to assign sources to slots at each training step; the cost matrix is $C_{ks} = -\log q_\phi(\vtheta_k^{\text{true}} \mid \vc_s)$.
Matched slots receive the flow NLL gradient; unmatched slots are trained toward the standard normal prior via the flow's base density ($\mathcal{L}_\text{prior}$, App.~\ref{app:lisa_training}).
An \emph{existence head} (MLP $\to$ logit) predicts whether each slot corresponds to a real source, trained with focal loss~\citep{Lin2017}.

\FloatBarrier
\section{Experiments}
\label{sec:experiments}

We evaluate evidence depletion in three settings of increasing realism: (i)~controlled synthetic benchmarks, (ii)~real-world audio mixtures (FUSS), and (iii)~LISA gravitational-wave inference.
The synthetic benchmarks are deliberately simple: they isolate the mechanism under controlled superposition.
In all settings, slot attention variants share identical architecture and loss functions, varying only the attention mechanism.
Figure~\ref{fig:synthetic} summarizes slot collapse across all five domains; detailed results follow in each subsection.

\begin{figure}[t]
\centering
\includegraphics[width=\linewidth]{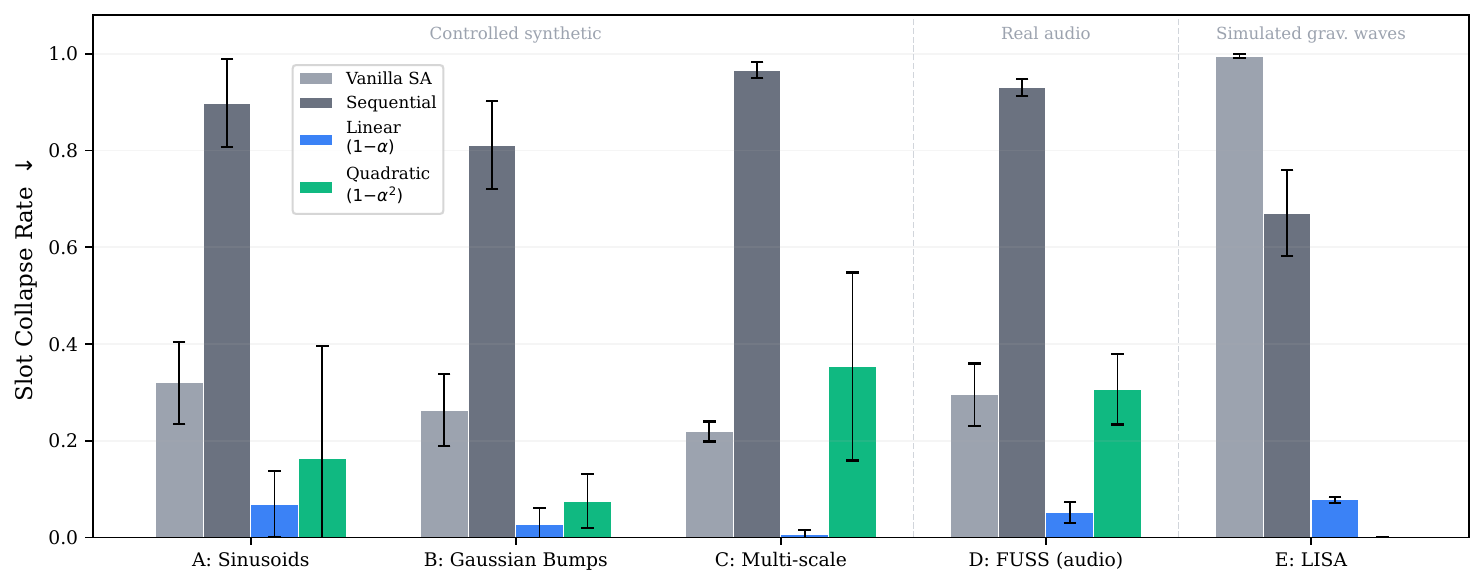}
\caption{Slot collapse across all five domains (error bars = std over 5 seeds for synthetic/FUSS, 3 seeds for LISA). Synthetic and FUSS use peak overlap rate; LISA uses max active overlap (Definition~\ref{def:collapse}). Evidence depletion (blue) consistently achieves the lowest collapse. Data-space variants in Table~\ref{tab:synthetic}.}
\label{fig:synthetic}
\end{figure}

\subsection{Synthetic benchmarks: isolating the attention mechanism}
\label{sec:synthetic}

These experiments use no auxiliary losses and isolate the attention mechanism in pure form.
We test on three signal types with variable source count ($K \sim \mathcal{U}\{1,\ldots,4\}$), one dominant source ($5$--$10\times$ amplitude), and additive Gaussian noise:
\textbf{(A)}~overlapping sinusoids (frequency-domain structure),
\textbf{(B)}~Gaussian bumps (time-domain, no periodicity),
\textbf{(C)}~multi-scale mixtures (sinusoids $\times$ Gaussian envelopes).
All use $N{=}256$ samples, 16 tokens, 5 slots, and are trained for 200 epochs on 5{,}000 samples (5 seeds each).
We compare vanilla SA~\citep{Locatello2020}, sequential SA (no residual tracking), and four evidence depletion forms -- binary (hard mask at $\alpha > 0.5$), linear $(1{-}\alpha)$, quadratic $(1{-}\alpha^2)$, and cubic $(1{-}\alpha^3)$ -- reporting the peak overlap rate of Definition~\ref{def:collapse}.
Table~\ref{tab:synthetic} shows the two best-performing forms; binary and cubic perform worse (Appendix~\ref{app:depletion_forms}).

\begin{table}[t]
\caption{Slot collapse rate (peak overlap rate; mean $\pm$ std, 5 seeds). Synthetic tasks use 200 epochs; FUSS uses 200 epochs on real audio. Data-space variants deplete tokens directly; evidence-space maintains a separate evidence variable.}
\label{tab:synthetic}
\centering\small
\begin{tabular}{@{}lcccc@{}}
\toprule
Method & Sinusoids & Gaussians & Multi-scale & FUSS \\
\midrule
Vanilla SA & $0.32 \pm 0.09$ & $0.26 \pm 0.07$ & $0.22 \pm 0.02$ & $0.29 \pm 0.06$ \\
Sequential SA & $0.90 \pm 0.09$ & $0.81 \pm 0.09$ & $0.97 \pm 0.02$ & $0.93 \pm 0.02$ \\
\midrule
Evidence linear $(1{-}\alpha)$ & $\mathbf{0.07 \pm 0.07}$ & $\mathbf{0.03 \pm 0.03}$ & $\mathbf{0.01 \pm 0.01}$ & $\mathbf{0.05 \pm 0.02}$ \\
Evidence quadratic $(1{-}\alpha^2)$ & $0.16 \pm 0.23$ & $0.08 \pm 0.06$ & $0.35 \pm 0.19$ & $0.31 \pm 0.07$ \\
\midrule
Data-space quad.\ (no bias) & $0.41 \pm 0.40$ & $0.23 \pm 0.39$ & $0.36 \pm 0.20$ & --- \\
Data-space quad.\ (with bias) & $0.18 \pm 0.30$ & $0.08 \pm 0.04$ & $0.35 \pm 0.20$ & --- \\
\bottomrule
\end{tabular}
\end{table}

Table~\ref{tab:synthetic} reports detailed results.
Three findings:
\begin{enumerate}[leftmargin=1.5em,itemsep=1pt]
    \item \textbf{Sequential SA collapses consistently} (0.81--0.97 across these tasks), confirming that sequential processing alone does not resolve slot collapse.
    \item \textbf{All evidence depletion forms reduce collapse} compared to vanilla and sequential.
    \item \textbf{The effect is consistent across domains} -- frequency, time, and heterogeneous signal types.
\end{enumerate}
Linear depletion $(1{-}\alpha)$ achieves the lowest collapse on these benchmarks.
We use quadratic $(1{-}\alpha^2)$ for the LISA experiments based on its softer exploration--commitment trade-off; whether linear would also suffice is an open question (Appendix~\ref{app:depletion_forms}).

\paragraph{How the two baselines fail differently.}
Sequential processing without residual tracking not only fails to resolve collapse but exacerbates it on synthetic benchmarks ($0.81$--$0.97$ vs.\ $0.22$--$0.32$ for vanilla), while on LISA it partially improves (Sec.~\ref{sec:lisa_ablation}).
On small problems, vanilla SA's parallel competition can find non-collapsed solutions through soft specialization, while sequential SA without residual tracking degenerates to a regime where early slots dominate allocation (``first slot wins'').
At scale, parallel competition fails entirely (NLL diverges, AngCorr $\approx$ random; Fig.~\ref{fig:ablation}), while sequential ordering gives each slot an independent gradient step -- but neither resolves collapse.

\paragraph{Evidence space vs.\ data space.}
We test an architecture-matched data-space variant -- identical sequential structure, identical quadratic $(1{-}\alpha^2)$ depletion, only the variable carrying the residual differs.
Tokens are depleted directly, $\vh_\ell \leftarrow \vh_\ell \cdot (1 - \alpha_{s\ell}^2)$, rather than via a separate evidence variable (Table~\ref{tab:synthetic}, bottom rows).
Without an attention bias, data-space depletion is unstable (std 0.39--0.40, with seeds collapsing entirely).
Adding a token-norm bias ($\gamma \log(\|\vh_\ell\| / \max_{\ell'}\|\vh_{\ell'}\|)$ in the attention logits) recovers performance to within noise of evidence-space depletion (e.g., 0.18 vs.\ 0.16 on sinusoids, 0.08 vs.\ 0.08 on Gaussians).
This isolates a clearer message than our initial framing: \emph{the multiplicative depletion structure together with an explicit attention bias is what prevents collapse.}
The multiplicative update enforces state evolution (tracking what has been explained), while the bias term ensures this state actively shapes attention allocation; both are required to break the symmetry of standard attention.
The critical ingredient is not the representation in which residuals are stored, but the presence of a mechanism that progressively redistributes explanatory capacity.
Whether the evidence-space formulation matters more in architectures with decoders that can compound subtraction errors is an open question.

\paragraph{Loss-based regularization is insufficient.}
We evaluate whether stronger loss-based regularization alone can resolve collapse, testing diversity losses ($w_\text{div}{=}5$), attention orthogonality losses ($w_\text{ortho}{=}3$), and their combination on both vanilla and sequential SA (Appendix~\ref{app:baselines}).
On vanilla SA, the best loss combination partially mitigates collapse ($0.06$/$0.11$/$0.06$ on sinusoids/Gaussians/multi-scale); evidence depletion (linear) matches or improves on every task ($0.07$/$0.03$/$0.01$) without any auxiliary loss.
On sequential SA, loss-based regularization fails entirely: even combining both losses, collapse remains at $0.30$--$0.43$.
Loss-based approaches treat the symptom; only mechanisms that modify the attention dynamics resolve collapse reliably.

\paragraph{Sensitivity to $\gamma$ and $\tau$.}
A sensitivity sweep over the bias strength $\gamma$ and softmax temperature $\tau$ on Synthetic-A is provided in Appendix~\ref{app:sensitivity}.

\subsection{Real-world audio mixtures (FUSS)}
\label{sec:fuss}

To test whether evidence depletion generalizes beyond synthetic signals, we evaluate on the Free Universal Sound Separation (FUSS) benchmark~\citep{FUSS2020}: real-world audio mixtures of 1--4 everyday sounds (speech, music, environmental noise) summed additively at 16\,kHz.
We crop to 4 seconds (64{,}000 samples), tokenize with 32 ChunkFFT tokens ($d{=}128$, 5 slots), and use spectral parameters (dominant frequency, RMS amplitude, spectral centroid) as ground truth.

Evidence depletion (linear) reduces collapse from $0.29 \pm 0.06$ (vanilla) to $0.05 \pm 0.02$ -- a ${\sim}6\times$ reduction on real audio, reaching the same range as the synthetic benchmarks (Table~\ref{tab:synthetic}).
Sequential SA collapses completely ($0.93 \pm 0.02$) despite identical architecture, hyperparameters, and training budget.
Unlike the synthetic benchmarks, FUSS contains spectrally complex real-world mixtures with heterogeneous source types; the persistence of collapse and its resolution by evidence depletion support that the failure mode is not an artifact of simplified signals but arises from memoryless attention under additive superposition.
Quadratic depletion ($0.31 \pm 0.07$) does not improve over vanilla on this benchmark: its softer update is insufficient in spectrally complex mixtures where stronger competition is required.
Full per-seed results are in Appendix~\ref{app:fuss_details}.

\subsection{Validation at scale: LISA inference}
\label{sec:lisa}

To test whether evidence depletion scales beyond simple benchmarks to a real scientific inference pipeline, we apply it to parameter estimation of overlapping galactic binary (GB) gravitational-wave sources for the Laser Interferometer Space Antenna~\citep{LISA2017}.
The setting is well suited to test the mechanism: GB signals contribute additively to every observation, exhibit large dynamic range across sources (${\sim}2$ orders of magnitude in strain amplitude), and overlap in frequency.
Each source is described by 8 parameters (frequency $f$, frequency derivative $\fdot$, amplitude $A$, inclination $\iota$, right ascension $\alpha$, declination $\delta$, polarization $\psi$, initial phase $\varphi_0$).
The signal model uses three noise-orthogonal detector channels~\citep{Tinto2021} sampled at $N{=}524{,}288$ points over 1\,year, colored noise per SciRDv1~\citep{SciRD2018}, $K_\text{resolve} \sim \mathcal{U}\{1,\ldots,4\}$ resolvable sources, and a Poisson($\lambda{=}3$) confusion foreground of sub-threshold sources.
Architecture and training details are in Appendix~\ref{app:lisa}; key attention hyperparameters are $\gamma{=}3$, $\tau{=}0.3$, $I{=}3$ refinement iterations.

\subsection{LISA ablation}
\label{sec:lisa_ablation}

The synthetic ablation (Sec.~\ref{sec:synthetic}, no auxiliary losses) isolates the attention mechanism in pure form.
This experiment tests whether the same mechanism scales to a complex inference pipeline where many components are necessary but only attention varies.
We train three models on identical data, architecture, and loss functions, varying only the slot attention mechanism:
\begin{itemize}[leftmargin=1.5em,itemsep=1pt]
    \item \textbf{Vanilla SA}: standard simultaneous competition (softmax over slots).
    \item \textbf{Sequential SA}: slots processed one at a time with softmax over tokens (not slots), no evidence masking or bias.
    \item \textbf{Evidence Depletion} (ours): sequential + evidence masking + log-evidence bias.
\end{itemize}

Table~\ref{tab:lisa_ablation} and Figure~\ref{fig:ablation} quantify the comparison across 3 seeds using the max active overlap of Definition~\ref{def:collapse}.
Vanilla SA collapses: overlap $0.99 \pm 0.00$, flow NLL diverges to $+7.1$, and CRPS remains at or above the prior, indicating failure to learn useful posteriors.
Sequential processing partially reduces overlap ($0.67 \pm 0.09$) but slots do not fully specialize.

Evidence depletion prevents collapse in this setting.
Overlap drops to $0.08 \pm 0.01$, and the model achieves the best CRPS on all reported parameters and the best flow NLL ($-6.0 \pm 0.4$); full per-parameter results are in Appendix~\ref{app:lisa_results}.
The effect is consistent across all seeds (Table~\ref{tab:lisa_ablation}).

\paragraph{Auxiliary losses.}
All three ablation models share the same loss function (App.~\ref{app:lisa}), including auxiliary terms ($\mathcal{L}_\text{div}$, $\mathcal{L}_\text{ortho}$, $\mathcal{L}_\text{ent}$) that also discourage redundant slot allocation.
The mechanism comparison therefore isolates the attention mechanism \emph{given} these auxiliary losses.
Evidence depletion's improvements on $f$, $A$, $\iota$, and $\alpha$ in Table~\ref{tab:lisa_ablation} are attributable to the attention mechanism alone (these parameters use only the primary 8D flow); the $\psi$ improvement may additionally benefit from the auxiliary 1D polarization flow, which is shared across all variants (App.~\ref{app:lisa}).
The synthetic ablation (Sec.~\ref{sec:synthetic}), which uses only the flow NLL and matching loss, further confirms that evidence depletion prevents collapse without any auxiliary regularization.
The full system trained to convergence on $10^6$ samples produces approximately calibrated posteriors (coverage within 5\,pp of nominal for constrained parameters; Table~\ref{tab:lisa_results}).
A qualitative visualization (Appendix~\ref{app:visualization}, Figure~\ref{fig:evidence}) confirms that evidence depletion produces genuinely distinct attention patterns on simulated LISA inputs (overlap 0.07 vs.\ 0.98 for vanilla).

\begin{table}[t]
\caption{LISA ablation ($10^5$ samples, 30 epochs, mean $\pm$ std over 3 seeds). Overlap = max pairwise cosine similarity between active slot attention vectors (Definition~\ref{def:collapse}).}
\label{tab:lisa_ablation}
\centering
\small
\begin{tabular}{@{}lcccccc@{}}
\toprule
Method & Overlap$\downarrow$ & NLL$\downarrow$ & CRPS $f$$\downarrow$ & CRPS $A$$\downarrow$ & CRPS $\iota$$\downarrow$ & CRPS $\alpha$$\downarrow$ \\
\midrule
Vanilla SA & $.99 {\scriptstyle\pm .00}$ & $+7.1 {\scriptstyle\pm 0.7}$ & $.34 {\scriptstyle\pm .02}$ & $.38 {\scriptstyle\pm .06}$ & $.32 {\scriptstyle\pm .08}$ & $.31 {\scriptstyle\pm .03}$ \\
Sequential & $.67 {\scriptstyle\pm .09}$ & $-4.1 {\scriptstyle\pm 0.2}$ & $.14 {\scriptstyle\pm .00}$ & $.13 {\scriptstyle\pm .00}$ & $.10 {\scriptstyle\pm .00}$ & $.12 {\scriptstyle\pm .00}$ \\
\textbf{Evid.\ Depl.} & $\mathbf{.08 {\scriptstyle\pm .01}}$ & $\mathbf{-6.0 {\scriptstyle\pm 0.4}}$ & $\mathbf{.05 {\scriptstyle\pm .01}}$ & $\mathbf{.09 {\scriptstyle\pm .01}}$ & $\mathbf{.08 {\scriptstyle\pm .00}}$ & $\mathbf{.09 {\scriptstyle\pm .00}}$ \\
\bottomrule
\end{tabular}
\end{table}

\begin{figure}[t]
\centering
\includegraphics[width=\linewidth]{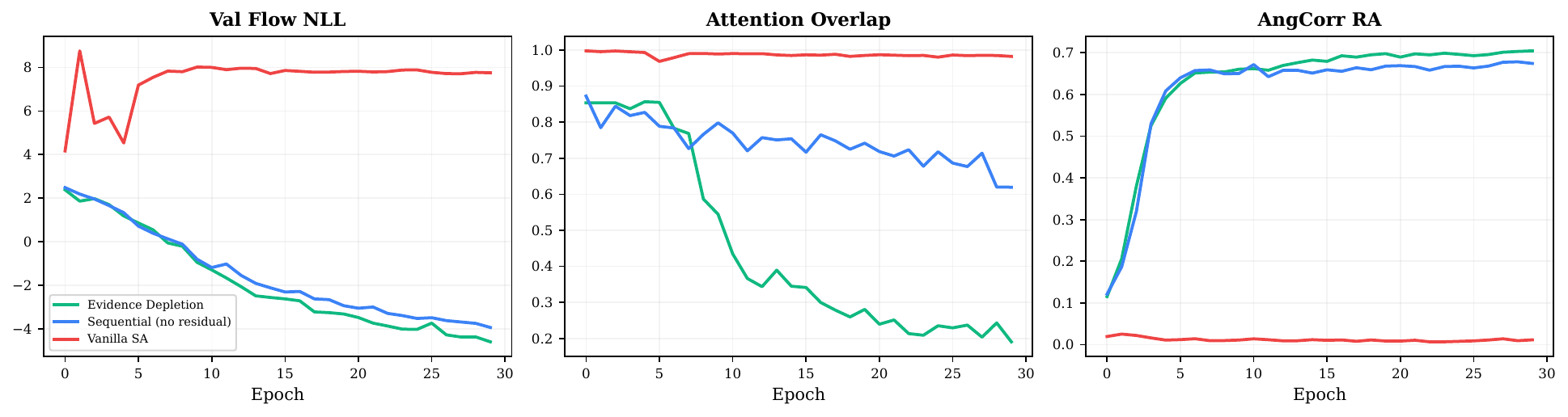}
\caption{LISA ablation: training dynamics ($10^5$ samples, 30 epochs). Vanilla SA (red) diverges with flow NLL $> 0$; sequential (blue) learns but maintains high attention overlap ($0.67$); evidence depletion (green) achieves low overlap and best metrics across all three panels.}
\label{fig:ablation}
\end{figure}

\section{Related work}
\label{sec:related}

\paragraph{Object-centric learning.}
Object-centric methods~\citep{Greff2019,Burgess2019} and slot attention~\citep{Locatello2020} with its extensions~\citep{Singh2022,Elsayed2022} decompose visual scenes into object representations.
SLATE~\citep{Singh2022SLATE} and Invariant Slot Attention~\citep{Biza2023} target representation fidelity and symmetry invariance respectively, and are orthogonal to slot collapse under additive superposition.
\citet{Krimmel2024} show that softmax normalization causes collapse under cardinality shift (more objects than seen during training); our finding under amplitude dominance is a distinct trigger with the same symptom, suggesting that slot attention's normalization is fragile under multiple conditions the original formulation did not anticipate.
MESH~\citep{Zhang2023MESH} recasts slot attention as optimal transport to break ties in object-centric settings; evidence depletion addresses a related symptom differently, changing the attention landscape after each allocation rather than directly modifying the assignment problem.
Genesis-V2~\citep{Engelcke2022} addresses slot collapse through architectural changes; our mechanism is grounded in residual dynamics and is readily applicable to sequential attention architectures.
Slot collapse under superposition can be viewed as a disentanglement failure due to insufficient inductive bias~\citep{Locatello2019}; evidence depletion provides the additional bias needed when observations are additive mixtures rather than spatial compositions.
Recent iterative and adaptive slot variants improve stability and flexibility; our work is complementary, addressing a structural failure mode under additive superposition.

\paragraph{Iterative subtraction and matching pursuit.}
\CLEAN{}~\citep{Hogbom1974} and matching pursuit~\citep{Mallat1993} decompose signals by greedily subtracting the best-matching component.
These methods modify the data through subtraction, propagating estimation errors into subsequent iterations; evidence depletion avoids this pathology by preserving the signal and modulating only attention allocation (Sec.~\ref{sec:evidence}).
In the hard-attention limit, evidence depletion reduces to progressive token masking rather than true signal subtraction.

\paragraph{Set prediction and iterative inference.}
DETR~\citep{Carion2020} introduced Hungarian matching for end-to-end set prediction.
We adopt the same matching strategy for variable-$K$ source assignment, extending it to posterior estimation via normalizing flows.
DETR-style parallel decoding with Hungarian matching is already incorporated in our vanilla slot attention baseline.
\citet{Marino2018} proposed iterative refinement of amortized posteriors; our sequential slot processing is complementary, with each slot capturing a different source.

\paragraph{Attention memory: sinks, inhibition of return, and coverage.}
In LLMs, \emph{attention sinks}~\citep{Xiao2024} describe concentration on uninformative tokens; slot collapse is conceptually related, where attention concentrates on dominant components across \emph{all} slots.
Evidence depletion has a structural analogue in \emph{inhibition of return}~\citep{Klein2000,Itti2001}, where previously attended locations are suppressed on a saliency map; unlike IOR, evidence depletion is trained end-to-end.
Coverage in neural machine translation~\citep{Tu2016} tracks cumulative attention $\sum_t \alpha_t$ to prevent over- or under-translation, sharing the principle of state over prior attention.
Evidence depletion differs structurally: coverage accumulates attention \emph{additively} over decoder time steps assuming near one-to-one alignment, while evidence depletion \emph{multiplicatively} decays per-token explanatory capacity across parallel slots in a many-to-many additive mixture.
Evidence also enters both the logit bias and the scaled keys/values, rather than logits alone.

\paragraph{Source separation and amortized inference.}
ICA~\citep{Hyvarinen2000} decomposes observations into independent sources under a linear mixing model; evidence depletion handles variable source count and produces approximately calibrated posteriors.
DINGO~\citep{Dax2021} and related work~\citep{Green2021} perform neural posterior estimation for single compact binary coalescence events.
\SlotFlow{} extends amortized inference to the multi-source setting with variable~$K$.
Existing LISA source-separation pipelines such as \texttt{GBMCMC}~\citep{Cornish2005,Littenberg2023} and global-fit approaches~\citep{Babak2008} target the full catalog with reversible-jump MCMC; our contribution is methodological rather than a replacement for these systems.

\section{Discussion and conclusion}
\label{sec:conclusion}

Standard attention is memoryless with respect to explained evidence.
Under additive superposition, this induces slot collapse: shared gradients drive multiple slots toward the same dominant component, leading to redundant allocation.
Evidence depletion resolves this by introducing a minimal form of state -- tracking residual explanatory capacity -- which breaks this symmetry and enforces progressive specialization.
Across synthetic signals, real-world audio mixtures, and LISA gravitational-wave inference, this mechanism consistently prevents collapse under identical architectures, data, and losses, where both parallel and sequential attention alone fail.
More broadly, attention-based decompositions are prone to failure when inputs do not admit a partition into components; compositional inference in such regimes benefits from explicit state over what has already been explained.
\textbf{Limitations.} Scaling to more complex settings likely requires hierarchical decomposition, as the current formulation operates at a single level of granularity.
The optimal depletion form is task-dependent and currently selected empirically; developing a principled selection criterion remains an open direction.


\bibliographystyle{plainnat}
\bibliography{slotflow_paper}


\newpage
\appendix

\section{LISA implementation details}
\label{app:lisa}

This appendix consolidates the architecture, flow context, and training details for the LISA experiments (Sec.~\ref{sec:lisa}, \ref{sec:lisa_ablation}).
All LISA data is simulated: the LISA mission is scheduled for launch in the mid-2030s and no observational data yet exists.
Waveforms are generated with JaxGB~\citep{JaxGB} using LISA orbital configurations from lisaorbits, with instrument noise following the Science Requirements Document v1~\citep{SciRD2018}.

\paragraph{Ablation validity.}
The LISA architecture is complex (21.7M parameters, dual-path encoder, 10-stream flow context, 10 loss terms), but the ablation in Sec.~\ref{sec:lisa_ablation} is designed to isolate the attention mechanism from this complexity.
All three variants (vanilla SA, sequential SA, evidence depletion) share \emph{exactly} the same tokenizer, source encoder, normalizing flow, loss function (including all auxiliary terms and weights), training data, hyperparameters, and random seed.
The only difference is the slot attention module: vanilla uses simultaneous softmax-over-slots competition, sequential processes slots one at a time with softmax over tokens, and evidence depletion adds the multiplicative update and log-evidence bias.
Any performance difference is therefore attributable to the attention mechanism under matched architecture, data, losses, and initialization.
The qualitative visualization in Appendix~\ref{app:visualization} further confirms this: the attention patterns themselves differ dramatically across variants (overlap 0.98 vs.\ 0.08), providing direct evidence that the mechanism, not the surrounding engineering, drives the improvement.
The architectural complexity reflects the difficulty of the underlying problem: LISA will observe millions of overlapping galactic binary signals in a single data stream of $524{,}288$ time samples, each source described by 8 correlated physical parameters with large dynamic range.
Extracting source-specific information from this mixture requires physics-informed feature extraction, temporal modulation encoding, and structured posterior estimation -- all orthogonal to the attention mechanism under study.

\subsection{Architecture}
\label{app:lisa_architecture}

\paragraph{Time-frequency tokenization.}
The signal ($3 \times 524{,}288$ detector-channel samples: A, E, T) is split into 4 overlapping segments, Hann-windowed, and FFT'd.
The CNN operates on 4 real-valued channels $[\text{Re}(\tilde{A}), \text{Im}(\tilde{A}), \text{Re}(\tilde{E}), \text{Im}(\tilde{E})]$ (the T channel is noise-dominated and omitted); Hz-based patch pooling produces 128 tokens ($\in \RR^{128}$) with learned positional encodings (Hz + seconds).

\paragraph{Sequential slot attention.}
$S = K_\text{max} + 1 = 5$ slots with evidence depletion ($\gamma = 3$, $\tau = 0.3$, $I = 3$ iterations, $\epsilon = 10^{-4}$).

\paragraph{Dual-path source encoder.}
Given each slot's frequency estimate $\hat{f}_s$ (produced by a differentiable regression head), a 64-bin spectral window is extracted using a Gaussian-weighted soft crop centered on $\hat{f}_s$ from 16 re-segmented time windows.
\textbf{Path 1} (physics-informed): 21 hand-crafted features per segment (power ratios, Stokes parameters, cross-phase, detrended phase) $\to$ temporal CNN $\to$ trunk MLP ($2048 \to 256$).
\textbf{Path 2} (learned): A 2D CNN on the raw 10-channel spectral window with \emph{frequency-only pooling} -- preserving the 16-step temporal orbital modulation that encodes sky position and polarization (Table~\ref{tab:encoder_ablation}).

\paragraph{Normalizing flow.}
12 RQ-spline coupling layers~\citep{Durkan2019}, 64 bins, hidden dim 640, physics-informed split: odd layers transform $(A, \iota, \psi, \varphi_0)$, even layers transform $(f, \fdot, \alpha, \delta)$.
Flow context: 778 dimensions from 10 feature streams (Table~\ref{tab:context}).
Total: 21.7M parameters.

\paragraph{Auxiliary 1D polarization flow.}
The polarization angle $\psi$ exhibits a 4-fold symmetry under $\psi \to \psi + \pi/2$ in LISA's antenna response, encoded by harmonics of period $\pi/2$.
A standard joint flow over all 8 parameters tends to average over the symmetric modes, yielding overly broad posteriors on $\psi$.
We therefore train a small auxiliary 1D normalizing flow on $\psi$ alone, conditioned on the same slot context, with the explicit harmonic encoding $(\sin(4\psi), \cos(4\psi))$ that respects this symmetry.
This auxiliary head is used only as a training-time regularizer through the $\mathcal{L}_\text{pol}$ loss term; the main posterior at evaluation comes from the primary 8-dimensional flow.
In our development runs, removing this term caused $\psi$ recovery to degrade to AngCorr $\approx 0$.

\begin{table}[h]
\caption{Encoder ablation: temporal-preserving (frequency-only pooling) vs.\ baseline (global pooling), matched epochs.}
\label{tab:encoder_ablation}
\centering
\small
\begin{tabular}{@{}lccc@{}}
\toprule
Metric & Baseline & Freq-only pool & $\Delta$ \\
\midrule
RA CRPS & 0.125 & 0.054 & $-57\%$ \\
Dec CRPS & 0.178 & 0.085 & $-52\%$ \\
Inc CRPS & 0.069 & 0.047 & $-32\%$ \\
Freq MAE ($\mu$Hz) & 0.51 & 0.26 & $-49\%$ \\
\bottomrule
\end{tabular}
\end{table}

\subsection{Flow context assembly}
\label{app:lisa_context}

The flow context vector concatenates 10 feature streams totaling 778 dimensions (Table~\ref{tab:context}).
SBI dropout is applied to the slot attention projection and frequency resolution streams to regularize against overconfidence.

\begin{table}[h]
\caption{Flow context streams (778-dim total).}
\label{tab:context}
\centering
\small
\begin{tabular}{@{}lrl@{}}
\toprule
Stream & Dim & Notes \\
\midrule
Slot attention proj.\ & 128 & SBI dropout 30\% \\
Source features proj.\ & 192 & SourceEncoder trunk \\
Modulation stats & 64 & Temporal statistics MLP \\
Anchor proj.\ & 32 & $\hat{f}_\text{norm}$, $\log P$ \\
Freq.\ resolution & 16 & SBI dropout 30\% \\
Frequency emb.\ & 32 & Frequency value \\
Angular emb.\ & 64 & AngularHead MLP \\
Angular skip & 10 & Raw sin/cos \\
Pol temporal & 48 & Stokes U + cross-phase \\
Spectro features & 192 & SpectrogramEncoder \\
\bottomrule
\end{tabular}
\end{table}

\subsection{Training}
\label{app:lisa_training}

\paragraph{Loss function.}
$\mathcal{L} = \mathcal{L}_\text{flow}(1) + \mathcal{L}_\text{exist}(3) + \mathcal{L}_\text{div}(0.5) + \mathcal{L}_{f,\text{attn}}(5) + \mathcal{L}_{f,\text{direct}}(50) + \mathcal{L}_\text{ang}(3) + \mathcal{L}_\text{prior}(0.1) + \mathcal{L}_\text{ortho}(0.3) + \mathcal{L}_\text{ent}(0.1) + \mathcal{L}_\text{pol}(1)$.
Weights in parentheses.
$\mathcal{L}_\text{flow}$: negative log-likelihood of the normalizing flow on matched (existing) slots.
$\mathcal{L}_\text{exist}$: focal loss~\citep{Lin2017} on existence-head logits, $\mathrm{FL}(p_t) = -\alpha_t (1-p_t)^\gamma \log p_t$ with $\gamma{=}2$, $\alpha{=}0.5$, label smoothing $0.01$.
$\mathcal{L}_{f,\text{direct}}$: MSE between predicted and true frequency in relative coordinates.
$\mathcal{L}_{f,\text{attn}}$: MSE on attention-weighted frequency estimates.
$\mathcal{L}_\text{ang}$: weighted MSE on angular sin/cos harmonic encodings (see below).
$\mathcal{L}_\text{div}$: mean squared cosine similarity between active slot embedding pairs,
$\mathcal{L}_\text{div} = \frac{1}{|\mathcal{P}|}\sum_{(i,j)\in\mathcal{P}} (\hat{\vs}_i^\top \hat{\vs}_j)^2$,
where $\hat{\vs}_i = \vs_i / \|\vs_i\|$ and $\mathcal{P}$ is the set of active slot pairs.
$\mathcal{L}_\text{ortho}$: mean absolute cosine similarity between active slot attention vectors, $\frac{1}{|\mathcal{P}|}\sum_{(i,j)\in\mathcal{P}} |\hat{\boldsymbol{\alpha}}_i^\top \hat{\boldsymbol{\alpha}}_j|$, optionally weighted by per-token signal power.
$\mathcal{L}_\text{ent}$: penalises high-entropy attention for active slots, $\mathcal{L}_\text{ent} = \frac{1}{|\mathcal{A}|}\sum_{s \in \mathcal{A}} H(\boldsymbol{\alpha}_s) / H_\text{max}$, where $H_\text{max} = \log L$.
$\mathcal{L}_\text{prior}$: flow NLL evaluated on standard-normal samples for unmatched (non-existing) slots, encouraging the flow to map unused contexts to a simple prior.
$\mathcal{L}_\text{pol}$: NLL of auxiliary 1D normalizing flow on $\psi$ alone, conditioned on slot context.
The auxiliary terms ($\mathcal{L}_\text{div}$, $\mathcal{L}_\text{ortho}$, $\mathcal{L}_\text{ent}$) discourage redundant slot allocation; they are held fixed across all ablation variants in Sec.~\ref{sec:lisa_ablation} so that only the attention mechanism varies.

\paragraph{Angular encoding.}
Harmonic multipliers: RA ($m{=}1$), Inc ($m{=}2$), Dec ($m{=}2$), Pol ($m{=}4$), Phase ($m{=}1$).
The $4\times$ harmonic for polarization matches the 4-fold symmetry of LISA's antenna pattern under $\psi \to \psi + \pi/2$; in preliminary experiments, $\sin(\psi)$ and $\sin(2\psi)$ encodings consistently failed to recover $\psi$.

\paragraph{Hyperparameters and resources.}
$10^5$ samples for the ablation comparison (30 epochs); $10^6$ samples for the full model (150 epochs).
Batch 128, lr $10^{-4}$, ReduceLROnPlateau (patience 10, factor 0.5), gradient clip 5.0, 4--8$\times$NVIDIA GH200 (120\,GB), $10^5$ validation samples, 21.7M parameters.

\section{LISA results at convergence}
\label{app:lisa_results}

All LISA parameters are normalized to $[0, 1]$ before computing CRPS (the expected CRPS for a uniform prior is $\approx 0.167$).
Table~\ref{tab:lisa_results} reports parameter recovery for the evidence depletion model trained to convergence on $10^6$ samples (150 epochs).
Frequency achieves CRPS $= 0.015$ (91\% compression versus the prior); amplitude, inclination, right ascension, and declination all reach CRPS $\leq 0.05$.
The initial phase $\varphi_0$ returns a prior-width posterior (CRPS $= 0.168 \approx 0.167$), as expected for an unconstrained parameter.
The frequency derivative $\fdot$ produces a broad posterior (CRPS $= 0.335$, above the prior baseline), indicating miscalibration from encoder compression (Appendix~\ref{app:fdot}).

Polarization ($\psi$) achieves partial recovery (AngCorr 0.72), limited by a physical degeneracy with sky position through LISA's orbital modulation~\citep{Cornish2003}.
Here, AngCorr denotes the angular correlation: the Pearson correlation between the predicted and true $(\sin(m\theta), \cos(m\theta))$ harmonic encodings, where $m$ is the parameter-specific harmonic multiplier (App.~\ref{app:lisa_training}).

\paragraph{Existence prediction and completeness.}
The existence head's 79\% accuracy and the 100\% source completeness numbers measure two different things and should not be conflated.
\emph{Source completeness} -- every true source is matched to at least one slot via the Hungarian assignment -- is 100\% across $K = 1$--$4$.
This metric is less stringent: with $S = 5$ slots and $K \leq 4$, completeness is essentially the model's ability to use slots non-redundantly, which evidence depletion enforces.
\emph{Existence-head accuracy} of 79\% measures how often the head correctly predicts which slots correspond to real sources.
Errors here skew toward false positives: the head occasionally flags an unmatched ``empty'' slot as active.
This is consistent with the focal-loss training and the practice of running with $S > K_\text{max}$.
For downstream use, this bias is the desirable direction -- false positives can be filtered by posterior width or amplitude threshold; false negatives would hide sources.
Frequency MAE degrades gracefully from $0.13\,\mu$Hz ($K{=}1$) to $0.27\,\mu$Hz ($K{=}4$).

\begin{table}[h]
\caption{LISA parameter recovery at convergence ($10^6$ samples, 150 epochs, realistic noise + confusion foreground). CRPS prior $\approx 0.167$ for uniform.}
\label{tab:lisa_results}
\centering
\small
\begin{tabular}{@{}lccccl@{}}
\toprule
Parameter & CRPS & AngCorr & Cov50\% & Cov90\% & Status \\
\midrule
Frequency $f$ & 0.015 & --- & 0.45 & 0.87 & Learned \\
Freq.\ deriv.\ $\fdot$ & 0.335 & --- & 0.49 & 0.89 & Miscal. \\
Amplitude $A$ & 0.050 & --- & 0.49 & 0.89 & Learned \\
Inclination $\iota$ & 0.042 & 0.80 & 0.48 & 0.89 & Learned \\
Right asc.\ $\alpha$ & 0.043 & 0.83 & 0.47 & 0.88 & Learned \\
Declination $\delta$ & 0.073 & 0.81 & 0.47 & 0.88 & Learned \\
Polarization $\psi$ & 0.147 & 0.72 & 0.51 & 0.89 & Partial \\
Phase $\varphi_0$ & 0.168 & --- & 0.51 & 0.91 & Prior \\
\midrule
\multicolumn{6}{@{}l}{\scriptsize Freq.\ MAE: $0.13\,\mu$Hz ($K{=}1$) to $0.27\,\mu$Hz ($K{=}4$). Exist.\ head accuracy 0.79; source completeness 1.00.} \\
\bottomrule
\end{tabular}
\end{table}

\section{Encoder compression and the limits of amortization}
\label{app:fdot}

A general property of any amortized inference system that compresses observations through an encoder is that information distributed thinly across many samples can be lost.
This appendix uses the gravitational-wave frequency derivative $\fdot$ as a concrete instance to illustrate the point; the observation generalizes beyond the LISA setting to any encoder-based simulation-based inference (SBI) system.

For sources with $\fdot \sim 10^{-16}$\,Hz/s, the phase accumulation over a 1-year observation is $\Delta \Phi = \pi \fdot T^2 \approx 0.3$\,rad, which is detectable in coherent template matching but is distributed essentially uniformly across $N = 524{,}288$ samples.
Our encoder reduces this to 16 detrended phase measurements per slot before passing to the source encoder; a subsequent $8{:}1$ compression discards the subtle quadratic phase curvature amid stronger angular gradients.
The model produces a broad posterior on $\fdot$ (CRPS $= 0.335$, above the uniform-prior baseline of $0.167$), indicating that encoder compression not only discards $\fdot$ information but introduces a bias -- the posterior is miscalibrated rather than merely uninformative.

The observation suggests a design principle for amortized multi-source inference: \emph{parameters whose information is distributed across many samples may require either matched-filter features or uncompressed representations}.
Resolving this is orthogonal to the slot collapse problem and applies to any encoder-based SBI architecture, especially when scaling to long signals or settings with weak, distributed evidence.

\section{Gradient analysis under additive superposition}
\label{app:gradient}

Consider the attention logit for slot $s$ at token $\ell$ containing two additive sources: $\vh_\ell = A_1 \vh^{(1)} + A_2 \vh^{(2)}$ with $A_1 \gg A_2$.
Through the key projection $\vk_\ell = W_k \vh_\ell$, the logit becomes $\vq_s \cdot \vk_\ell = A_1 (\vq_s \cdot W_k \vh^{(1)}) + A_2 (\vq_s \cdot W_k \vh^{(2)})$.
The gradient of the softmax attention with respect to slot parameters $\vq_s$ is (up to the scaling constant $1/(\tau\sqrt{d})$):
\[
\frac{\partial \alpha_{s\ell}}{\partial \vq_s} \propto \alpha_{s\ell}(1 - \alpha_{s\ell}) \cdot \vk_\ell = \alpha_{s\ell}(1 - \alpha_{s\ell}) \cdot (A_1 W_k \vh^{(1)} + A_2 W_k \vh^{(2)})\,.
\]
When $A_1 \gg A_2$, the gradient \emph{direction} is dominated by the $A_1$ term regardless of the scaling or softmax normalization.
Crucially, this gradient direction is the same for all slots: since slots are initialized from a shared prior and observe the same tokens, they receive updates pointing in the same $A_1$-dominated direction.
After several iterations, all slots converge to the same representation -- not because they lack capacity, but because the optimization landscape presents identical gradients to each slot.
This is the cross-slot symmetry that produces shared fixed points and slot collapse.

Evidence depletion breaks this symmetry by reducing $e_\ell$ after the first slot attends, which attenuates $\vk_\ell = W_k(e_\ell \vh_\ell)$ and removes the gradient dominance of $A_1$ for subsequent slots.
Each slot now faces a different gradient landscape, preventing convergence to a shared fixed point.

\section{Sensitivity to $\gamma$ and $\tau$}
\label{app:sensitivity}

We sweep the bias strength $\gamma \in \{1, 3, 10\}$ and softmax temperature $\tau \in \{0.1, 0.3, 1.0\}$ on Synthetic-A (sinusoids), reporting peak overlap rate (Definition~\ref{def:collapse}, mean $\pm$ std over 3 seeds).

\begin{table}[h]
\caption{Sensitivity to $\gamma$ and $\tau$ on Synthetic-A (sinusoids), using quadratic depletion $(1{-}\alpha^2)$. Peak overlap rate (mean $\pm$ std, 3 seeds). Default values used in the main paper: $\gamma{=}3$, $\tau{=}0.3$.}
\label{tab:sensitivity}
\centering\small
\begin{tabular}{@{}lccc@{}}
\toprule
& $\tau = 0.1$ & $\tau = 0.3$ & $\tau = 1.0$ \\
\midrule
$\gamma = 1$  & $0.35 \pm 0.45$ & $\mathbf{0.03 \pm 0.01}$ & $0.76 \pm 0.17$ \\
$\gamma = 3$  & $0.31 \pm 0.42$ & $0.06 \pm 0.05$ & $0.69 \pm 0.33$ \\
$\gamma = 10$ & $\mathbf{0.02 \pm 0.01}$ & $0.03 \pm 0.01$ & $0.26 \pm 0.25$ \\
\bottomrule
\end{tabular}
\end{table}

The dominant trend is the temperature $\tau$: moderate values ($\tau{=}0.3$) yield consistently low collapse across all $\gamma$, while high temperature ($\tau{=}1.0$) washes out the bias and causes collapse in most settings.
At low temperature ($\tau{=}0.1$), the method works well with strong bias ($\gamma{=}10$, collapse $0.02$) but becomes unstable at weaker bias (std $> 0.4$), likely because sharp attention concentrates on noise fluctuations before the bias can redirect.
The default configuration ($\gamma{=}3$, $\tau{=}0.3$) lies in the stable low-collapse regime; strong bias ($\gamma{=}10$) at $\tau{=}0.3$ performs comparably ($0.03$).
These results confirm that evidence depletion is robust across a moderate range of hyperparameters and does not require careful tuning.

\section{Collapse metrics}
\label{app:metrics}

We define the two collapse metrics from Definition~\ref{def:collapse} formally.
Let $\alpha_s \in \RR^L$ denote the attention vector of slot $s$ over $L$ tokens.

\paragraph{Peak overlap rate (synthetic and FUSS).}
For each multi-source input with $K \geq 2$ ground-truth sources, we consider the first $K$ slots (ordered by slot index) and compute each slot's peak token $p_s = \arg\max_\ell \alpha_{s\ell}$.
The input exhibits collapse if any two of these slots share a peak: $\exists\, s \neq s' \in \{1,\ldots,K\}: p_s = p_{s'}$.
The peak overlap rate is the fraction of multi-source inputs that exhibit collapse:
\[
    \text{PeakOverlap} = \frac{1}{|\{i : K_i \geq 2\}|} \sum_{i:\, K_i \geq 2} \mathbf{1}\!\left[\exists\, s \neq s' \in \{1,\ldots,K_i\} : p_s^{(i)} = p_{s'}^{(i)}\right].
\]

\paragraph{Max active overlap (LISA).}
Let $\mathcal{A} = \{s : \sigma(\text{exist-logit}(s)) > 0.5\}$ be the set of predicted active slots.
For each input with $|\mathcal{A}| \geq 2$, compute the maximum pairwise cosine similarity between active slot attention vectors:
\[
    \text{MaxOverlap} = \max_{\substack{s, s' \in \mathcal{A} \\ s \neq s'}} \frac{\alpha_s \cdot \alpha_{s'}}{\|\alpha_s\|\,\|\alpha_{s'}\|}\,,
\]
averaged over inputs. This captures the worst-case pair of collapsing slots and is more informative than peak overlap when $L$ is large (128 tokens for LISA) and attention patterns can be similar without sharing a single peak.

\section{Additional depletion forms}
\label{app:depletion_forms}

Table~\ref{tab:synthetic} in the main text reports the two best-performing depletion forms (linear and quadratic). For completeness, Table~\ref{tab:all_forms} includes binary and cubic variants.

\begin{table}[h]
\caption{All depletion forms: slot collapse rate (mean $\pm$ std, 5 seeds, 200 epochs).}
\label{tab:all_forms}
\centering\small
\begin{tabular}{@{}lccc@{}}
\toprule
Method & Sinusoids & Gaussians & Multi-scale \\
\midrule
Binary (hard mask at $\alpha > 0.5$) & $0.30 \pm 0.27$ & $0.41 \pm 0.19$ & $0.71 \pm 0.26$ \\
Linear $(1{-}\alpha)$ & $\mathbf{0.07 \pm 0.07}$ & $\mathbf{0.03 \pm 0.03}$ & $\mathbf{0.01 \pm 0.01}$ \\
Quadratic $(1{-}\alpha^2)$ & $0.16 \pm 0.23$ & $0.08 \pm 0.06$ & $0.35 \pm 0.19$ \\
Cubic $(1{-}\alpha^3)$ & $0.11 \pm 0.05$ & $0.26 \pm 0.32$ & $0.52 \pm 0.13$ \\
\bottomrule
\end{tabular}
\end{table}

Binary depletion performs worst due to the discontinuous threshold; cubic is intermediate. The monotonic trend (linear $>$ quadratic $>$ cubic $>$ binary) on Gaussians and multi-scale suggests that stronger depletion generally helps, though cubic outperforms quadratic on sinusoids ($0.11$ vs.\ $0.16$).

\section{Loss-based baselines}
\label{app:baselines}

We test whether stronger loss-based regularization can resolve slot collapse without modifying the attention mechanism.
Six baselines combine vanilla or sequential SA with stronger diversity losses ($w_\text{div}{=}5.0$, vs.\ the default $0.5$), attention orthogonality losses ($w_\text{ortho}{=}3.0$), or both.
All use identical architecture, data, and training budget as the main ablation (200 epochs, 5 seeds).

\begin{table}[h]
\caption{Loss-based baselines: slot collapse rate (mean $\pm$ std, 5 seeds, 200 epochs). Reference: evidence depletion (linear) achieves 0.07/0.03/0.01 on these tasks without auxiliary losses.}
\label{tab:baselines}
\centering\small
\begin{tabular}{@{}llccc@{}}
\toprule
SA variant & Loss config & Sinusoids & Gaussians & Multi-scale \\
\midrule
Vanilla & $w_\text{div}{=}5$ & $0.14 \pm 0.04$ & $0.20 \pm 0.04$ & $0.10 \pm 0.02$ \\
Vanilla & $w_\text{ortho}{=}3$ & $0.08 \pm 0.02$ & $0.10 \pm 0.03$ & $0.07 \pm 0.02$ \\
Vanilla & Both & $0.06 \pm 0.03$ & $0.11 \pm 0.07$ & $0.06 \pm 0.01$ \\
\midrule
Sequential & $w_\text{div}{=}5$ & $0.74 \pm 0.16$ & $0.68 \pm 0.05$ & $0.85 \pm 0.07$ \\
Sequential & $w_\text{ortho}{=}3$ & $0.45 \pm 0.06$ & $0.51 \pm 0.08$ & $0.49 \pm 0.04$ \\
Sequential & Both & $0.30 \pm 0.04$ & $0.36 \pm 0.03$ & $0.43 \pm 0.02$ \\
\midrule
\multicolumn{2}{@{}l}{\textbf{Evidence linear} (no aux.\ loss)} & $\mathbf{0.07 \pm 0.07}$ & $\mathbf{0.03 \pm 0.03}$ & $\mathbf{0.01 \pm 0.01}$ \\
\bottomrule
\end{tabular}
\end{table}

Loss-based regularization partially reduces collapse on vanilla SA (e.g., $0.32 \to 0.06$ on sinusoids with both losses) but completely fails on sequential SA ($0.90 \to 0.30$ at best).
Evidence depletion achieves lower collapse \emph{without any auxiliary loss}, confirming that the mechanism -- not the regularization -- is the operative ingredient.

\paragraph{Scope of comparisons.}
This paper studies slot collapse under \emph{additive superposition}, where every token contains contributions from all sources simultaneously.
This regime is fundamentally different from object-centric vision (CLEVR, Multi-dSprites, MOVi), where different objects occupy disjoint spatial regions and tokens can be uniquely assigned to components.
Methods designed for the spatial regime -- such as MESH~\citep{Zhang2023MESH}, which modifies the assignment problem via optimal transport -- address a different failure mode (tie-breaking under spatial ambiguity) and are not directly applicable to the additive setting.
We therefore compare against baselines that share the same problem structure: vanilla and sequential slot attention (which face additive superposition but lack residual tracking) and loss-based regularization (which attempts to resolve collapse without modifying the attention dynamics).
This controlled comparison isolates the contribution of the attention mechanism itself, which is the paper's core contribution.

\section{FUSS benchmark details}
\label{app:fuss_details}

The Free Universal Sound Separation (FUSS) benchmark~\citep{FUSS2020} contains real-world audio mixtures of 1--4 everyday sounds (speech, music, environmental noise, etc.) summed additively at 16\,kHz in 10-second clips.
We crop to 4 seconds (64{,}000 samples), tokenize with 32 ChunkFFT chunks (FFT, project to $d{=}128$), and extract 3 ground-truth parameters per source: dominant frequency, RMS amplitude, and spectral centroid (all normalized to $[0,1]$).
Table~\ref{tab:fuss_full} reports per-seed results.

\begin{table}[h]
\caption{FUSS results: slot collapse rate per seed (200 epochs, 5{,}000 train / 1{,}000 val, 32 tokens, $d{=}128$).}
\label{tab:fuss_full}
\centering\small
\begin{tabular}{@{}lcccccc@{}}
\toprule
Method & Seed 0 & Seed 1 & Seed 2 & Seed 3 & Seed 4 & Mean $\pm$ std \\
\midrule
Vanilla SA & 0.22 & 0.38 & 0.34 & 0.31 & 0.22 & $0.29 \pm 0.06$ \\
Sequential SA & 0.94 & 0.91 & 0.94 & 0.91 & 0.95 & $0.93 \pm 0.02$ \\
Evidence linear & 0.09 & 0.04 & 0.04 & 0.04 & 0.04 & $\mathbf{0.05 \pm 0.02}$ \\
Evidence quadratic & 0.42 & 0.27 & 0.25 & 0.36 & 0.23 & $0.31 \pm 0.07$ \\
\bottomrule
\end{tabular}
\end{table}

Quadratic depletion $(1{-}\alpha^2)$ does not improve over vanilla on FUSS ($0.31$ vs.\ $0.29$) -- it is effectively inert on spectrally complex audio.
Linear depletion $(1{-}\alpha)$ achieves a ${\sim}6\times$ reduction ($0.29 \to 0.05$).

\paragraph{Practical guidance on depletion form.}
Across all synthetic benchmarks and FUSS, linear depletion matches or outperforms quadratic.
We use quadratic for LISA based on its softer update; a direct comparison on LISA is left to future work.
We recommend linear $(1{-}\alpha)$ as the starting point; quadratic $(1{-}\alpha^2)$ is worth trying when the downstream task requires calibrated density estimation (e.g., normalizing flows) rather than just source separation.
Developing a principled selection criterion -- e.g., based on the spectral entropy of the input or the number of sources -- is an open direction.

\section{Visualizing evidence depletion on a simulated LISA input}
\label{app:visualization}

\begin{figure}[t!]
\centering
\includegraphics[width=\linewidth,trim=0 0 0 15,clip]{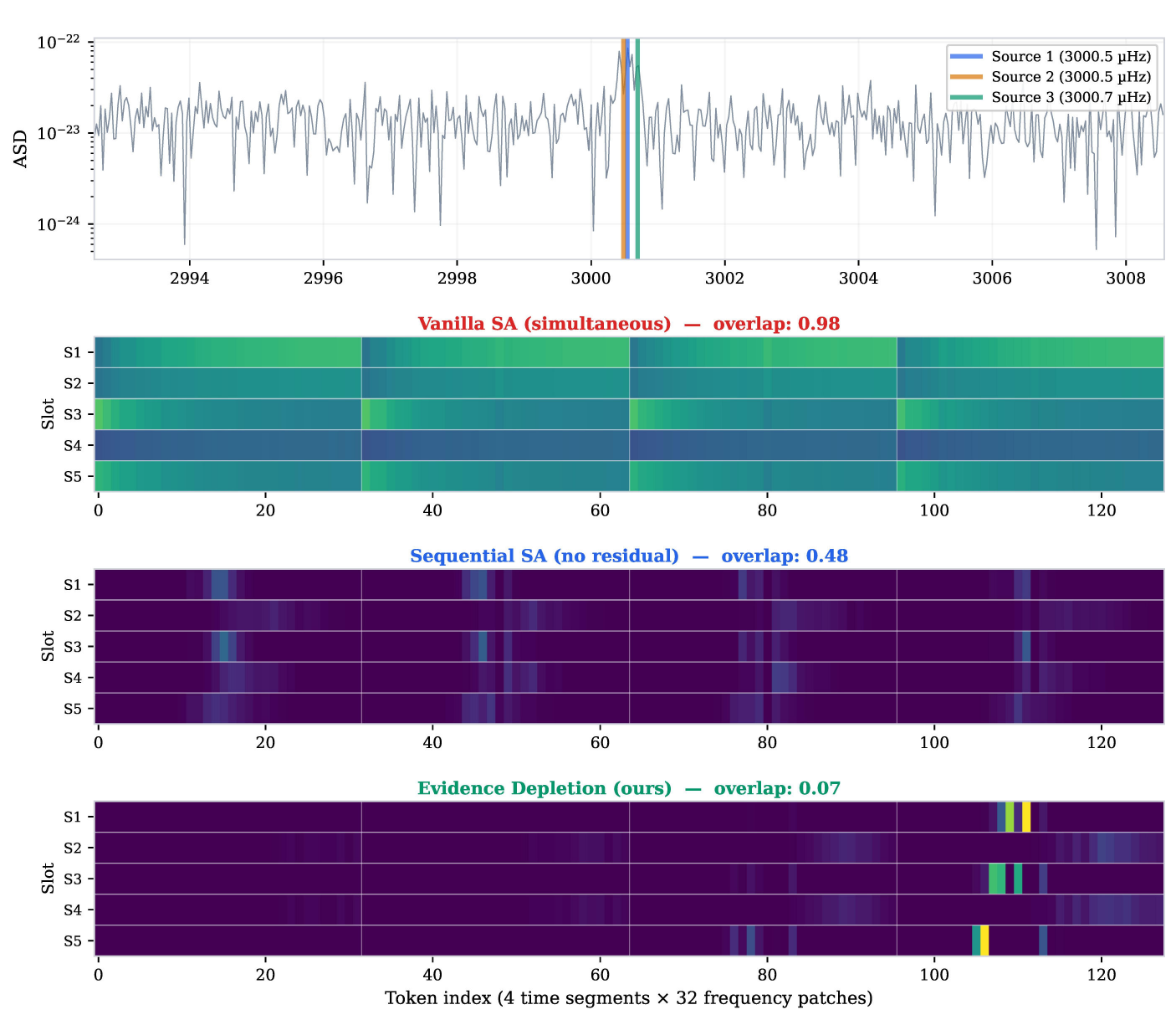}
\caption{\textbf{Evidence depletion resolves slot collapse by breaking shared-gradient symmetry.}
\textit{Top:} Amplitude spectral density of a $K{=}3$ input with three nearly co-frequency sources ($\Delta f < 1\,\mu$Hz), forming a single unresolved feature.
\textit{Bottom:} Attention heatmaps (5 slots $\times$ 128 tokens) for three variants.
\textbf{Vanilla SA:} All slots attend nearly uniformly (overlap 0.98). Because every slot observes the same input, gradients are dominated by the same component, driving slots to identical solutions (collapse).
\textbf{Sequential SA:} Processing slots one at a time introduces ordering but not state; all slots still attend to the same tokens (overlap 0.48), reflecting persistent shared gradients.
\textbf{Evidence depletion:} After each slot, attended tokens are downweighted via a residual evidence variable, changing the effective input seen by subsequent slots. This forces subsequent slots to extract complementary structure from the same tokens (overlap 0.07), as the residual evidence alters the effective input and breaks shared gradient directions, here exploiting temporal modulation induced by LISA's orbital motion rather than frequency separation.
Token indices: 4 temporal segments $\times$ 32 frequency patches; segment boundaries shown as white lines.}
\label{fig:evidence}
\end{figure}

Figure~\ref{fig:evidence} illustrates the gradient-symmetry argument of Sec.~\ref{sec:collapse} on a simulated LISA input: a $K{=}3$ sample with three sources at nearly identical frequencies ($\Delta f < 1\,\mu$Hz), all contributing to a single unresolved spectral feature.
The attention heatmaps (5 slots $\times$ 128 tokens, organized as 4 temporal segments $\times$ 32 frequency patches) show how each variant handles this shared-input regime.

Vanilla SA (overlap 0.98) exhibits exactly the failure mode predicted by the memoryless-attention analysis: all slots distribute attention uniformly because the dominant component attracts every slot equally, and there is no mechanism to redirect subsequent slots.
Sequential SA (overlap 0.48) provides some gradient diversity through ordering but does not break the shared fixed points -- all slots still attend to the same tokens.
With evidence depletion (overlap 0.07), the picture changes qualitatively: once the bias suppresses already-attended frequency bins, remaining slots discover complementary structure along LISA's orbital modulation axis -- the sources are unresolved in frequency but separable in time via direction-dependent antenna patterns.
This demonstrates that evidence depletion does not merely reduce overlap numerically but redirects attention to genuinely complementary features, consistent with the structural argument that breaking gradient symmetry enables slots to specialize.

\end{document}